\documentclass[11pt]{article}

\usepackage[T1]{fontenc}
\usepackage[utf8]{inputenc}
\usepackage{lmodern}

\usepackage{geometry}
\usepackage{indentfirst}    
\usepackage{microtype}
\usepackage{hyperref}       
\usepackage{url}            
\usepackage{booktabs}       
\usepackage{amsfonts, amsmath, amsthm, amssymb} 
\usepackage{mathtools}
\usepackage{nicefrac}       
\usepackage{xcolor}         
\usepackage{enumitem}
\usepackage{algorithm}
\usepackage{algpseudocode}
\usepackage{tikz}      
\usepackage{subcaption}
\usepackage{comment}
\usepackage{natbib}

\newtheorem{theorem}{Theorem}
\newtheorem{lemma}[theorem]{Lemma}

\newtheorem{assumption}{Assumption}
\theoremstyle{definition}

\newtheorem{remark}[theorem]{Remark}

\newcommand{\R}{\mathbb{R}}
\newcommand{\E}{\mathbb{E}}
\newcommand{\Pp}{\mathbb{P}}

\newcommand{\Ind}{\mathbf{1}}

\newcommand{\thst}{\theta^{\star}}
\newcommand{\hth}{\widehat{\theta}}

\newcommand{\D}{\Delta}

\newcommand{\Hist}{\mathcal{H}}

\title{Optimal Contextual Pricing under \\
Agnostic Non-Lipschitz Demand}

\author{
Jianyu Xu \\
Carnegie Mellon University\\
Pittsburgh, PA 15213 \\
\texttt{jianyux@andrew.cmu.edu}
\and
Yu-Xiang Wang \\
University of California San Diego\\
La Jolla, CA 92093 \\
\texttt{yuxiangw@ucsd.edu}
}

\date{}

\begin{document}

\maketitle

\begin{abstract}
We study contextual dynamic pricing with linear valuations and bounded-support agnostic noise, whose induced demand curve may be non-Lipschitz with arbitrary jumps and atoms.
Such discontinuities break the cross-context interpolation arguments used by smooth-demand pricing algorithms, while the best previous method achieved only $\widetilde O(T^{3/4})$ regret.
We propose Conservative-Markdown Redirect-UCB Pricing, a polynomial-time algorithm that combines randomized parameter estimation, conservative residual-grid probing, and confidence-based one-step redirection.
Our algorithm achieves $\widetilde O(T^{2/3})$ optimal regret, matching the known lower bounds of \citet{kleinberg2003value} up to logarithmic factors and improving over the previous upper bound of \citet{xu2022towards}.
Under stochastic well-conditioned contexts, this closes the long-existing open regret gap in linear-valuation contextual pricing under agnostic non-Lipschitz noise distribution.
\end{abstract}

\newpage

\section{Introduction}
\label{sec:intro}
Learning from censored feedback is a recurring challenge across machine learning: the learner makes a decision, but observes only a partial outcome rather than the full reward landscape. Contextual dynamic pricing is a canonical instance of this challenge. A seller observes a feature vector describing the current buyer or product, posts a price in real time, and observes only whether the buyer purchases. The seller never observes the buyer's valuation itself. The goal is to learn a pricing policy whose cumulative revenue is close to that of a clairvoyant policy that knows both the linear valuation parameter and the noise distribution.

A standard model writes the valuation as a linear contextual component plus an idiosyncratic noise:

\vspace{0.5em}
\noindent
\fbox{\parbox{0.95\textwidth}{\textbf{Contextual pricing.} For $t=1,2,\ldots,T$:
        \small
        \begin{enumerate}[leftmargin=*,align=left,itemsep=0pt,topsep=2pt]
            \item A context $x_t\in\mathbb{R}^{d}$ is observed.
            \item The buyer has private valuation $y_t=x_t^\top\theta^{\star}+\xi_t$.
            \item The seller posts a price $p_t\in[0,B]$.
            \item The seller observes only $o_t=\mathbf{1}\{p_t\le y_t\}$.
            \item The seller receives revenue $p_t o_t$.
        \end{enumerate}
    }
}
\vspace{0.5em}

The statistical difficulty of the problem depends strongly on the regularity of the noise distribution. If the noise distribution is known, parametric, log-concave, Lipschitz, or sufficiently smooth, existing algorithms can exploit this structure to transfer information across nearby residual prices. In many pricing environments, however, such regularity is not realistic. Buyers may cluster around psychological thresholds, reservation values may contain mass points, and heterogeneous subpopulations may create sharp jumps in the demand curve. In these cases, the survival function 
may be discontinuous and non-Lipschitz. A small residual perturbation can then change the purchase probability by a constant amount, invalidating interpolation arguments based on smoothness.

This paper studies contextual dynamic pricing in this fully agnostic non-Lipschitz regime. We impose no density, Lipschitz, smoothness, log-concavity, single-crossing, or parametric condition on the noise distribution beyond mild boundedness and i.i.d. assumptions. The noise may have arbitrary atoms and jumps. The formal setup and assumptions are given in Section~\ref{sec:setup}.

\paragraph{Main contribution.}
We introduce \emph{Conservative-Markdown Redirect-UCB Pricing}, a polynomial-time algorithm that achieves $\widetilde O(T^{2/3})$ regret for linear-valuation contextual pricing with fully agnostic bounded noise. 
This rate is optimal in its $T$-dependence up to logarithmic factors. A $T^{2/3}$ lower bound already holds for an interior non-contextual subclass \citep{kleinberg2003value} that can be embedded into our stochastic-context model, and \citet{xu2022towards} establish the same barrier for contextual pricing under Lipschitz demand. Therefore, our result closes the \(T^{3/4}\) versus \(T^{2/3}\) gap for agnostic linear-valuation pricing under the stochastic contextual setting studied here.

\paragraph{Scope of agnosticism.} Our agnosticism is with respect to the residual demand curve: we impose no Lipschitz continuity, smoothness, density, log-concavity, or parametric structure on the valuation noise distribution, but only boundedness to avoid boundary clipping.  

\paragraph{Algorithmic idea.}
The algorithm reduces contextual pricing to confidence-bound learning on a one-dimensional residual grid. A uniformly randomized first stage estimates the linear valuation component; subsequent prices use a conservative markdown so that every binary observation becomes a one-sided probe of two adjacent survival levels, a relation that remains valid even when the demand curve has jumps. In the adaptive stage, an upper-confidence rule probes uncertain grid points directly and redirects to the left once their confidence radius reaches the grid scale, so discontinuity costs are paid only during a limited exploration period.

\paragraph{Technical contribution.}
The analysis compares the oracle value directly with the optimistic score of a maintained residual-grid index, rather than estimating demand at oracle-selected points. Direct probes pay grid, confidence, and adjacent-jump costs, while redirected probes pay only a grid-scale price loss; summing these terms with $\Delta\asymp T^{-1/3}$ yields $\widetilde O(T^{2/3})$ regret. The same residual grid has size $\widetilde O(T^{1/3})$, giving a direct polynomial-time implementation.

\begin{table}[t]
  \caption{Comparison of dynamic pricing regret bounds under different noise assumptions.}
  \label{tab:comparison}
  \centering
  \begin{tabular}{lll}
    \toprule
    Noise distribution assumptions & Best known regret & Reference \\
    \midrule
    $O(1/T)$ standard deviation & $O(\log T)$ & \citet{cohen2020feature_journal} \\
    known and log-concave & $O(\log T)$ & \citet{javanmard2019dynamic} \\
    Parametric Family & $\widetilde O(\sqrt T)$ & \citet{ban2021personalized} \\
    Unknown, second-order smooth & $\widetilde O(T^{2/3})$ & \citet{luo2022contextual} \\
    Unknown, $k$th-order smooth & $\widetilde O(T^{\frac{2k+1}{4k-1}})$ & \citet{fan2024policy} \\
    Unknown, Lipschitz & $\widetilde O(T^{2/3})$ & \citet{tullii2024improved} \\
    Unknown, Lipschitz & $\widetilde \Omega(T^{2/3})$ & \citet{xu2022towards} \\
    Agnostic, bounded & $\widetilde O(T^{3/4})$ & \citet{xu2022towards} \\
    Agnostic, non-contextual & $\Omega(T^{2/3})$ & \citet{kleinberg2003value} \\
    \textbf{Agnostic, bounded} & $\bf{\widetilde O(T^{2/3})}$ & Theorem~\ref{thm:main}, \textbf{this work} \\
    \bottomrule
  \end{tabular}
\end{table}

Table~\ref{tab:comparison} summarizes the regret landscape for linear-valuation contextual pricing under different noise assumptions. The rest of the paper is organized as follows. Section~\ref{sec:related} reviews related work. Section~\ref{sec:setup} formalizes the model and assumptions. Section~\ref{sec:alg} presents the algorithm. Section~\ref{sec:main} gives the main regret theorem and proof roadmap. Section~\ref{sec:experiments} reports numerical experiments, and Section~\ref{sec:conclusion} discusses limitations, ethical considerations, and conclusions. Full proofs appear in the appendix.

\section{Related Work}
\label{sec:related}
We review the closest literature on contextual dynamic pricing with unknown demand. Additional background on parametric pricing, non-contextual posted-price learning, confidence-bound methods, and adjacent contextual pricing models appears in Appendix~\ref{sec:app_more_related_works}.

\paragraph{Contextual pricing with smooth or Lipschitz noise.}
A large body of work obtains strong regret guarantees by imposing regularity on the noise distribution. \citet{javanmard2019dynamic} study dynamic pricing with known log-concave noise, where the revenue function is well behaved enough for greedy pricing after parameter learning. \citet{ban2021personalized} obtain $\widetilde O(\sqrt T)$ regret under parametric demand assumptions. \citet{luo2022contextual} achieve $\widetilde O(T^{2/3})$ regret under second-order smoothness of the noise distribution through an explore-then-UCB strategy. \citet{fan2024policy} exploit higher-order smoothness and introduce a semiparametric parameter-estimation method based on uniformly randomized prices, which we also use in Stage~1. \citet{tullii2024improved} obtain $\widetilde O(T^{2/3})$ regret under Lipschitz noise via cross-context learning. These works use smoothness to transfer demand information between nearby residuals. Such transfer is unavailable in the agnostic setting considered here, where a residual perturbation of arbitrarily small size may cross an atom and change demand by a constant amount.

\paragraph{Fully agnostic contextual pricing.}
The closest work is \citet{xu2022towards}, who study linear-valuation contextual pricing with an unknown noise distribution and no smoothness assumptions. Their D2-EXP4 algorithm discretizes a large policy class and applies an adversarial bandit method, obtaining $\widetilde O(T^{3/4})$ regret for adversarial context sequences. They also prove a $\widetilde\Omega(T^{2/3})$ lower bound for contextual pricing under a Lipschitz subclass, which applies to the larger agnostic model. Our algorithm works directly on a one-dimensional residual grid rather than enumerating a high-dimensional policy class. The conservative markdown creates a stable one-sided observation model on this grid, and the redirect-UCB rule controls discontinuity-induced bias through confidence-radius accounting. This yields the optimal $\widetilde O(T^{2/3})$ regret rate in polynomial time.

\paragraph{Lower bounds.}
The \(\widetilde O(T^{2/3})\) rate is tight in its dependence on \(T\). Xu and Wang~\citep{xu2022towards} prove a \(\widetilde\Omega(T^{2/3})\) lower bound for contextual pricing under Lipschitz noise, showing that the \(T^{2/3}\) barrier already appears in a regular subclass of demand functions. In addition, the classical non-contextual lower bound of \citet{kleinberg2003value} can be embedded into our model by using a constant linear valuation together with independent dummy covariates to satisfy the full-rank stochastic-context condition. The hard valuation distributions can be placed in an interior price interval, so bounded support, zero-mean normalization, and the price-buffer condition are satisfied after the usual recentering of the noise into the intercept. 
Appendix~\ref{app:lower-bound-alignment} gives the formal embedding.

\section{Problem Setup}
\label{sec:setup}
We study contextual dynamic pricing with linear valuations and binary purchase feedback. At each round $t=1,\ldots,T$, the seller observes a context $x_t\in\R^d$. The customer's valuation is $y_t=\langle x_t,\thst\rangle+\xi_t$, where $\thst\in\R^d$ is unknown and $\xi_t$ is an unobserved noise. The seller posts a price $p_t\in[0,B]$ based on past observations and observes only $o_t=\Ind\{p_t\le y_t\}$.

We write the demand curve in residual coordinates through the survival function $S(w):=\Pp(\xi_t\ge w)$. This convention handles atoms without additional notation: if $F(w)=\Pp(\xi_t\le w)$ is the usual right-continuous CDF, then $S(w)=1-F(w^-)$, so $\E[o_t\mid x_t,p_t]=S(p_t-\langle x_t,\thst\rangle)$ for arbitrary noise distributions. We extend $S$ outside the support by setting $S(w)=1$ for $w\le -c$ and $S(w)=0$ for $w>c$.

The expected revenue at context $x$ and price $p$ is $\pi(x,p):=pS(p-\langle x,\thst\rangle)$. Let
\begin{equation}\label{eq:regret}
    \mathcal R_T:=\sum_{t=1}^T\left\{\max_{p\in[0,B]}\pi(x_t,p)-\pi(x_t,p_t)\right\},
    \qquad
    R_T:=\E[\mathcal R_T].
\end{equation}
For compactness, write $u_t:=\langle x_t,\thst\rangle$ and $V_t(w):=(u_t+w)S(w)$. The oracle value in round $t$ is $V_t^\star:=\max_{w\in[-c,c]}V_t(w)$; residuals outside $[-c,c]$ cannot improve revenue under the bounded-support model.

\subsection{Assumptions}\label{subsec:assumptions}

We impose boundedness and identifiability conditions on the linear contextual component, while leaving the demand curve otherwise unrestricted: $S$ may be discontinuous, non-Lipschitz, and may contain arbitrary jumps and atoms.

\begin{assumption}[Bounded support, exogeneity, and normalization]\label{ass:bdd}
There exist known constants $B_x,B_\theta,c,B<\infty$ such that, almost surely, $\|x_t\|_2\le B_x$, $\|\thst\|_2\le B_\theta$, $\xi_t\in[-c,c]$, and $0\le y_t\le B$. The noises $\{\xi_t\}_{t=1}^T$ are i.i.d. from an unknown distribution $F$ and independent of the context sequence. We use the location normalization $\E[\xi_t]=0$.
\end{assumption}

The valuation bound $0\le y_t\le B$ is used only in Stage~1 to convert uniformly randomized prices into an unbiased linear signal. The zero-mean normalization is without loss of generality when the context contains an intercept, since a nonzero noise mean can be absorbed into the intercept coefficient. The independence condition can be weakened to $\E[\xi_t\mid x_t]=0$ for the Stage~1 argument.

\begin{assumption}[Stochastic contexts]\label{ass:iid}
The contexts $\{x_t\}_{t=1}^T$ are i.i.d. from a distribution $\mathcal D_x$ supported on $\{x:\|x\|_2\le B_x\}$, and $\lambda_{\min}(\E[x_tx_t^\top])\ge\lambda_0>0$.
\end{assumption}

Assumption~\ref{ass:iid} ensures that the linear parameter is identifiable from the randomized Stage~1 observations.

\begin{assumption}[Buffered admissible prices]\label{ass:buffer}
There exists a constant $\kappa>0$ such that $\langle x_t,\thst\rangle-c\ge\kappa$ and $\langle x_t,\thst\rangle+c\le B$ almost surely. We assume the horizon is large enough that $2\D\le\kappa$, where $\D$ is the Stage~1 target accuracy in Algorithm~\ref{alg:main}; smaller horizons are absorbed into the problem-dependent constant.
\end{assumption}

\paragraph{Scope of the assumptions.}
Bounded support and bounded valuations make the residual grid finite and justify the randomized-price regression in Stage~1. The i.i.d. full-rank context condition is used only for estimating the linear component, while the price-buffer condition avoids boundary clipping in the markdown probes. None of these assumptions imposes continuity, density, Lipschitzness, smoothness, slope, or margin conditions on \(F\) or \(S\); arbitrary atoms and jumps remain allowed.


\subsection{Notation}\label{subsec:notation}

Let $\hat u_t:=\langle x_t,\hth\rangle$ be the predicted linear value after Stage~1, and define the Stage~1 good event
\begin{equation}\label{eq:theta-good-event}
    \mathcal E_\theta:=\left\{\max_{1\le t\le T}|\hat u_t-u_t|\le\D\right\}.
\end{equation}
Given $\D$, define a residual grid with spacing $2\D$ by $M:=\lceil c/\D\rceil$ and $w_i:=-c+2i\D$ for $i=0,1,\ldots,M$. The queried index set is $\mathcal I:=\{0,1,\ldots,M-1\}$. We use the boundary convention $w_{-1}:=w_0-2\D$ and $S(w_{-1}):=1$.

For each queried index $j\in\mathcal I$, define the adjacent jump height $\alpha_j:=S(w_{j-1})-S(w_j)\ge0$. Since $S$ is non-increasing and bounded in $[0,1]$, $\sum_{j=0}^{M-1}\alpha_j\le1$. This quantity captures the maximum discontinuity cost associated with probing the interval $[w_{j-1},w_j]$. Let $\Hist_t$ denote the sigma-algebra generated by observations through round $t$. We use $\widetilde O(\cdot)$ to hide factors polylogarithmic in $T$ and polynomial in $(d,B_x,B_\theta,B,c,1/\lambda_0,1/\kappa)$. We focus on the horizon dependence and treat $d$ and other problem constants as fixed.

\section{Algorithm}
\label{sec:alg}
We present Conservative-Markdown Redirect-UCB Pricing (CMRUP) in Algorithm~\ref{alg:main}. The algorithm has three stages. Stage~1 estimates the linear valuation parameter. Stage~2 collects initial demand observations on a residual grid. Stage~3 uses upper confidence bounds on the residual grid, together with a one-step redirect rule, to choose prices adaptively.

\paragraph{Markdown probes.}
For a queried index $j\in\mathcal I=\{0,\ldots,M-1\}$, define the probe price
\begin{equation}\label{eq:probe-action}
    p^+_{t,j}:=\operatorname{clip}_{[0,B]}\bigl(\hat u_t+w_{j+1}-3\D\bigr).
\end{equation}
On $\mathcal E_\theta$ and under Assumption~\ref{ass:buffer}, clipping does not occur in Stages~2--3. Since $|\hat u_t-u_t|\le\D$, the effective residual of this probe satisfies
\[
    p^+_{t,j}-u_t
    =w_{j+1}-3\D+(\hat u_t-u_t)
    \in [w_{j-1},w_j].
\]
Thus the conditional purchase probability of probe $j$ obeys
\begin{equation}\label{eq:probe-sandwich-main}
    S(w_j)
    \le
    \E[o_t\mid \Hist_{t-1},x_t,\text{probe }j]
    \le
    S(w_{j-1}).
\end{equation}
This deterministic sandwich is the basic observation model used by the algorithm on the residual grid.

\paragraph{Empirical means and confidence radii.}
For each queried index $j$, let $n_j(t)$ be the number of times before round $t$ that probe action $j$ has been played, and let $\hat m_j(t)$ be the empirical mean of the corresponding binary observations. If $n_j(t)=0$, set $\hat m_j(t)=0$. Define
\begin{equation}\label{eq:ucb-radius}
    b_j(t):=C_{\rm ucb}\sqrt{\frac{\log T}{\max\{1,n_j(t)\}}}.
\end{equation}
The constant $C_{\rm ucb}$ is chosen large enough for the uniform concentration event in Lemma~\ref{lem:raw-confidence}.

\paragraph{Optimistic scores.}
At a Stage~3 round $t$, each queried index receives the score
\begin{equation}\label{eq:redirect-score}
    U_{t,j}
    :=
    \bigl(\hat u_t+w_{j+1}+\D\bigr)
    \min\{1,\hat m_j(t)+b_j(t)\}.
\end{equation}
The first factor is an optimistic price multiplier for residuals in $[w_j,w_{j+1}]$, and the second factor is an upper confidence estimate for $S(w_j)$.

\begin{algorithm}[htbp]
\caption{Conservative-Markdown Redirect-UCB Pricing (CMRUP)}
\label{alg:main}
\begin{algorithmic}[1]
\State \textbf{Input:} horizon $T$; bounds $B,B_x,B_\theta,c,\lambda_0$; constants $C_\theta,C_{\rm ucb}$.
\State $T_1\gets\lceil T^{2/3}\rceil$, $T_w\gets\lceil T^{2/3}\rceil$.
\Statex
\State \textbf{Stage 1: parameter estimation.}
\For{$t=1,\ldots,T_1$}
    \State Post $p_t\sim\mathrm{Unif}[0,B]$ independently; observe $o_t$.
\EndFor
\State Set $z_t\gets B o_t$ for $t\le T_1$.
\State Compute
\[
    \hth:=\left(\sum_{t=1}^{T_1}x_tx_t^\top\right)^\dagger
    \sum_{t=1}^{T_1}z_tx_t,
\]
where $\dagger$ denotes the Moore--Penrose pseudoinverse. 
On the covariance-good event used in the proof, this equals the usual inverse.
\State Set
\[
    \D\gets C_\theta\frac{B B_x^2}{\lambda_0}
    \sqrt{\frac{d\log T}{T_1}}.
\]
\State Construct the grid $w_i=-c+2i\D$, $i=0,\ldots,M$, where $M=\lceil c/\D\rceil$.
\State Initialize $n_j\gets0$ and $\hat m_j\gets0$ for all $j\in\mathcal I=\{0,\ldots,M-1\}$.
\Statex
\State \textbf{Stage 2: warmup probes.}
\For{$t=T_1+1,\ldots,T_1+T_w$}
    \State Observe $x_t$ and set $\hat u_t\gets\langle x_t,\hth\rangle$.
    \State Sample $J_t\sim\mathrm{Unif}(\mathcal I)$.
    \State Post $p_t\gets\operatorname{clip}_{[0,B]}(\hat u_t+w_{J_t+1}-3\D)$; observe $o_t$.
    \State Update $n_{J_t}$ and $\hat m_{J_t}$ using $o_t$.
\EndFor
\Statex
\State \textbf{Stage 3: adaptive pricing.}
\For{$t=T_1+T_w+1,\ldots,T$}
    \State Observe $x_t$ and set $\hat u_t\gets\langle x_t,\hth\rangle$.
    \For{$j\in\mathcal I$}
        \State $b_j\gets C_{\rm ucb}\sqrt{\log T/\max\{1,n_j\}}$.
        \State $U_{t,j}\gets(\hat u_t+w_{j+1}+\D)\min\{1,\hat m_j+b_j\}$.
    \EndFor
    \State $j_t\in\arg\max_{j\in\mathcal I}U_{t,j}$.
    \If{$b_{j_t}>\D$}
        \State $a_t\gets j_t$. \Comment{direct probe}
    \ElsIf{$j_t=0$}
        \State $a_t\gets0$. \Comment{boundary case}
    \Else
        \State $a_t\gets j_t-1$. \Comment{one-step redirect}
    \EndIf
    \State Post $p_t\gets\operatorname{clip}_{[0,B]}(\hat u_t+w_{a_t+1}-3\D)$; observe $o_t$.
    \State Update $n_{a_t}$ and $\hat m_{a_t}$ using $o_t$.
\EndFor
\end{algorithmic}
\end{algorithm}

\subsection{Mechanics of the Algorithm}\label{subsec:alg-design}

\paragraph{Residual-grid observations.}
The $3\D$ markdown combines one grid step of length $2\D$ with the prediction-error allowance $\D$. As a result, a probe indexed by $j$ generates an observation whose conditional mean lies between $S(w_j)$ and $S(w_{j-1})$. The possible mismatch between these two values is exactly the adjacent jump height $\alpha_j$.

\paragraph{Upper confidence search.}
The score $U_{t,j}$ is an upper confidence estimate of the revenue associated with residual interval $[w_j,w_{j+1}]$. Lemma~\ref{lem:optimism} shows that, on the joint good event, the maximum score upper bounds the oracle revenue in every Stage~3 round. The score uses only observations collected from probes that were actually played.

\paragraph{One-step redirect.}
Directly probing index $j$ may incur a discontinuity cost proportional to $\alpha_j$. The algorithm directly probes $j$ only while $b_j>\D$, which limits the number of direct probes at each index. Once $b_j\le\D$, the selected index is executed by probing $j-1$ instead. This redirected probe has at least the demand level needed to validate the selected score, and the incurred price loss is only of order $\D$.

\paragraph{Role of the warmup stage.}
The warmup stage initializes the residual-grid observations and gives every grid point a chance to be sampled before the adaptive stage. It is not the source of the $T^{2/3}$ rate: the Stage~3 confidence analysis depends only on counts of actually played probes and would also work with infinite initial radii and forced first visits. We use $T_w=\lceil T^{2/3}\rceil$ because this cost is of the same order as the target regret and keeps the algorithmic description simple. A shorter initialization or a Stage~3-only initialization rule can be analyzed with the same counting argument.

\section{Regret Analysis}
\label{sec:main}
We state the regret guarantee for Algorithm~\ref{alg:main} and summarize the proof. Full proofs are in Appendix~\ref{sec:app-proofs}.

\begin{theorem}[Cumulative regret]\label{thm:main}
Under Assumptions~\ref{ass:bdd}--\ref{ass:buffer}, Algorithm~\ref{alg:main} satisfies
\[
    R_T\le \widetilde O(T^{2/3}).
\]
The guarantee allows arbitrary bounded-support noise distributions, including distributions whose CDF or survival function has atoms, jumps, and no Lipschitz or smoothness regularity.
\end{theorem}

The constants hidden in $\widetilde O(\cdot)$ may depends polynomially on $(d,B_x,B_\theta,B,c,1/\lambda_0,1/\kappa)$. Please kindly find these detailed dependence in Appendix~\ref{app:main-proof}.

\paragraph{Proof architecture.}
The proof has four ingredients. First, uniformly randomized prices in Stage~1 produce $\hth$ with $\max_t|\hat u_t-u_t|\le\D=\widetilde O(T^{-1/3})$ with high probability. Second, conservative markdown converts this prediction guarantee into the deterministic sandwich $S(w_j)\le m_{t,j}\le S(w_{j-1})$ whenever probe index $j$ is played. Third, empirical confidence bounds on adaptively collected probe means imply optimism: the largest score $\max_j U_{t,j}$ upper bounds the oracle revenue in each Stage~3 round. Finally, direct probes and redirected probes are charged separately. Direct probes may pay the adjacent jump cost $\alpha_j$, but each index is directly probed only until its confidence radius reaches the grid scale; redirected probes avoid this jump cost and lose only $O(\D)$ in price.

\paragraph{Key lemmas.}
The first lemma gives the prediction guarantee for the linear component.

\begin{lemma}[Parameter estimation]\label{lem:stage1}
Under Assumptions~\ref{ass:bdd}--\ref{ass:iid}, if $T_1=\lceil T^{2/3}\rceil$ and $\D$ is chosen as in Algorithm~\ref{alg:main}, then with probability at least $1-T^{-2}$, we have $\max_{1\le t\le T}|\hat u_t-u_t|\le \D.$
\end{lemma}

\begin{proof}[Proof sketch]
For $p_t\sim\mathrm{Unif}[0,B]$ and $y_t\in[0,B]$, the signal $z_t:=Bo_t$ satisfies $\E[z_t\mid x_t]=\E[y_t\mid x_t]=\langle x_t,\thst\rangle$, using exogeneity and $\E[\xi_t]=0$. Standard bounded-design least-squares concentration, together with the empirical covariance lower bound implied by Assumption~\ref{ass:iid}, gives $\|\hth-\thst\|_2\lesssim (B B_x/\lambda_0)\sqrt{d\log(T)/T_1}$ with high probability. Multiplying by $B_x$ and using the definition of $\D$ yields the claim. See Appendix~\ref{app:stage1-proof} for a detailed proof.
\end{proof}

\begin{lemma}[Markdown sandwich]\label{lem:markdown-sandwich}
On $\mathcal E_\theta$, for every Stage~2--3 round in which probe index $j$ is played, its conditional purchase probability $m_{t,j}:=\E[o_t\mid \Hist_{t-1},x_t,\text{probe }j]$ satisfies
\[
    S(w_j)\le m_{t,j}\le S(w_{j-1}).
\]
\end{lemma}

\begin{proof}[Proof sketch]
The un-clipped probe price is $\hat u_t+w_{j+1}-3\D$. Assumption~\ref{ass:buffer} prevents clipping, and $|\hat u_t-u_t|\le\D$ places the realized residual in $[w_{j-1},w_j]$. Monotonicity of $S$ gives the sandwich. See Appendix~\ref{app:sandwich-proof} for more details.
\end{proof}

\begin{lemma}[Uniform confidence for adaptive probes]\label{lem:raw-confidence}
There exists a sufficiently large constant $C_{\rm ucb}$ such that, with probability at least $1-T^{-2}$, simultaneously for all Stage~3 rounds $t$ and all $j\in\mathcal I$,
\[
    |\hat m_j(t)-\bar m_j(t)|\le \frac12 b_j(t),
    \qquad
    b_j(t):=C_{\rm ucb}\sqrt{\frac{\log T}{\max\{1,n_j(t)\}}}.
\]
On the same event, $S(w_j)\le \hat m_j(t)+b_j(t)\le S(w_{j-1})+2b_j(t)$.
\end{lemma}

\begin{proof}[Proof sketch]
For each queried index, enumerate the stopping times at which that index is actually probed. The centered observations at these times form a bounded martingale-difference sequence under the stopped filtration, even though the probing times are adaptive and may include redirected plays. A Hoeffding--Azuma bound for each index and sample size, followed by a union bound over $j$ and $m\le T$, yields the uniform event. Combining it with Lemma~\ref{lem:markdown-sandwich} gives the two one-sided inequalities. See Appendix~\ref{app:confidence-proof} for a complete proof.
\end{proof}

\begin{lemma}[Optimism of the score]\label{lem:optimism}
On $\mathcal E_\theta\cap\mathcal E_{\rm conf}$, every Stage~3 round $t$ satisfies
\[
    V_t^\star\le \max_{j\in\mathcal I}U_{t,j}.
\]
\end{lemma}

\begin{proof}[Proof sketch]
Let $w_t^\star$ be an optimal residual and choose $q$ such that $w_t^\star\in[w_q,w_{q+1}]$. Such a $q\le M-1$ exists because $w_0=-c$ and $w_M\ge c$; if $w_t^\star=c$, we take $q=M-1$. On $\mathcal E_\theta$, the price multiplier in $U_{t,q}$ dominates $u_t+w_t^\star$, while monotonicity gives $S(w_q)\ge S(w_t^\star)$. Lemma~\ref{lem:raw-confidence} gives $\hat m_q+b_q\ge S(w_q)$, so the score of $q$ upper bounds $V_t^\star$. By Assumption~\ref{ass:buffer} and \(2\Delta\le\kappa\), all candidate price
levels appearing in the optimistic scores are nonnegative. See Appendix~\ref{app:optimism-proof}.
\end{proof}

\begin{lemma}[One-round regret]\label{lem:one-round}
On $\mathcal E_\theta\cap\mathcal E_{\rm conf}$, let $j_t$ be the score maximizer in Stage~3.
\begin{enumerate}[label=(\roman*),topsep=2pt,itemsep=2pt]
\item If Algorithm~\ref{alg:main} is in direct-probe mode and plays $j_t$, then
\[
    V_t^\star-\pi(x_t,p_t)
    \le
    C\D+B\alpha_{j_t}+CB b_{j_t}(t).
\]
\item If Algorithm~\ref{alg:main} is in redirect mode, then
\[
    V_t^\star-\pi(x_t,p_t)
    \le
    C_B\D,
\]
where $C_B$ depends only on $B$.
\end{enumerate}
\end{lemma}

\begin{proof}[Proof sketch]
By optimism, the oracle value is at most the selected score. In direct-probe mode, the selected score can exceed the revenue of the played probe only through the grid-scale price slack, the adjacent jump $\alpha_{j_t}$, and the confidence radius $b_{j_t}$. The truncation $\min\{1,\hat m_j+b_j\}$ can only reduce the score. In redirect mode, the algorithm probes $j_t-1$ when $j_t\ge1$; this action has conditional demand at least $S(w_{j_t-1})$, which covers the selected score's demand component up to the confidence radius. Since redirect mode occurs only when $b_{j_t}\le\D$, the remaining gap is $O(\D)$. The boundary case $j_t=0$ uses $S(w_{-1})=S(w_0)=1$. See Appendix~\ref{app:one-round-proof} for more details.
\end{proof}

\begin{lemma}[Cumulative Stage~3 regret on the good event]\label{lem:stage3-good}
On $\mathcal E_\theta\cap\mathcal E_{\rm conf}$, the cumulative regret in Stage~3 is $\widetilde O(T^{2/3})$.
\end{lemma}

\begin{proof}[Proof sketch]
Redirect rounds contribute at most $O(\D)$ each by Lemma~\ref{lem:one-round}, hence $O(T\D)=\widetilde O(T^{2/3})$. Direct-probe rounds contribute $\D$, $b_{j_t}(t)$, and $\alpha_{j_t}$. The first term is again $O(T\D)$. The confidence terms satisfy the standard counting bound $\sum_{\rm probe}b_{j_t}(t)\le \widetilde O(\sqrt{MT})=\widetilde O(T^{2/3})$ since $M=O(\D^{-1})$. For the jump terms, index $j$ is directly probed only while $b_j>\D$, so it is directly probed at most $O(\log T/\D^2)$ times; therefore $\sum_{\rm probe}\alpha_{j_t}\le O(\log T/\D^2)\sum_j\alpha_j=\widetilde O(T^{2/3})$. See Appendix~\ref{app:cumulative-proof}.
\end{proof}

\paragraph{Proof sketch of Theorem~\ref{thm:main}.}
Stages~1 and~2 each last $O(T^{2/3})$ rounds and have per-round regret at most $B$. Lemma~\ref{lem:stage3-good} controls Stage~3 on the joint good event, whose complement has probability $O(T^{-2})$ and contributes at most $BT\cdot O(T^{-2})=O(1/T)$. Thus $R_T\le\widetilde O(T^{2/3})$. Full details are in Appendix~\ref{app:main-proof}.

\section{Numerical Experiments}
\label{sec:experiments}
\begin{figure}[t]
    \centering
    \includegraphics[width=\textwidth]{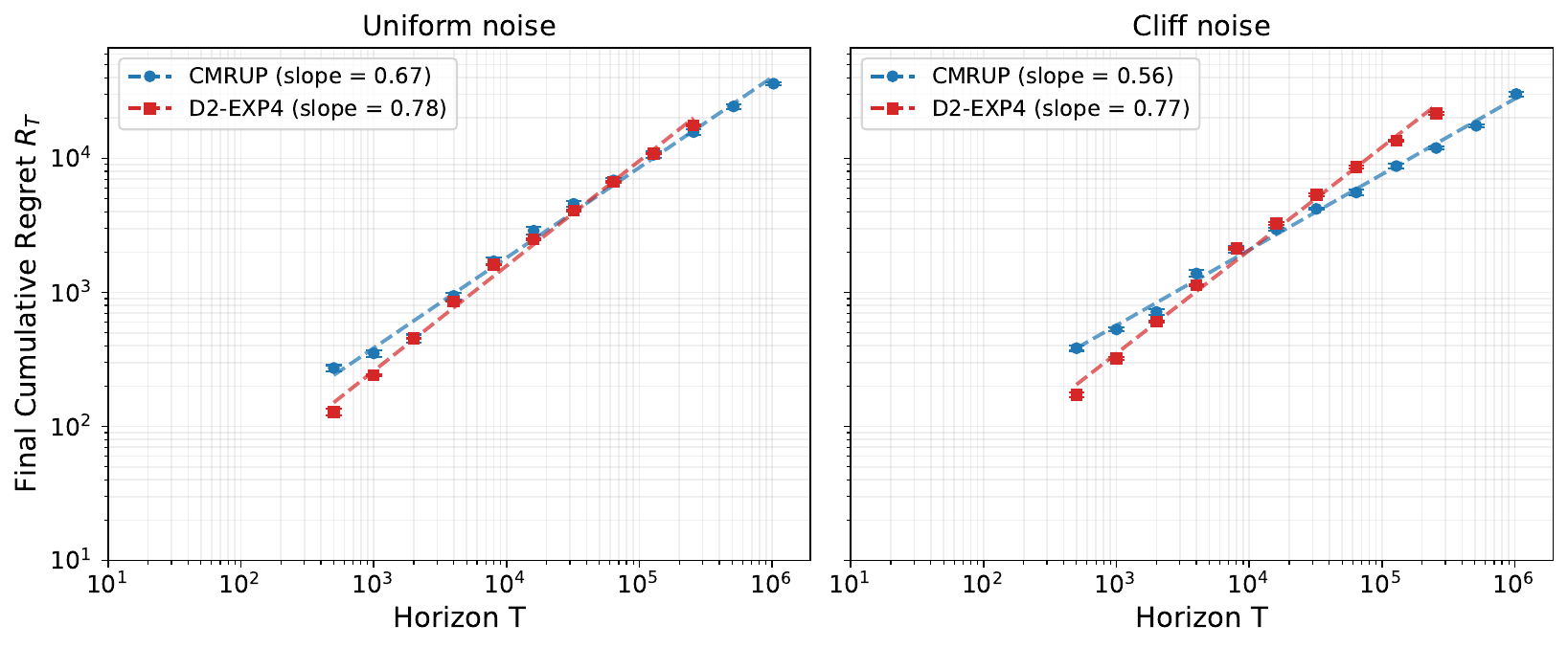}
    \caption{Final cumulative pseudo-regret $R_T$ on synthetic contextual-pricing instances with smooth uniform noise and discontinuous cliff noise. CMRUP is averaged over $10$ independent seeds, while D2-EXP4 is averaged over $5$ seeds due to its substantially higher computational cost. Error bars show standard errors across seeds. Dashed lines are all-point least-squares fits on the log-log scale, using all available horizons for each algorithm.}
    \label{fig:loglog}
\end{figure}

We report synthetic experiments designed to verify the final-regret scaling predicted by Theorem~\ref{thm:main}. The experiments compare CMRUP with the agnostic D2-EXP4 baseline of \citet{xu2022towards} on both a smooth Lipschitz noise distribution and a discontinuous non-Lipschitz distribution containing an atom. The primary quantity of interest is the final cumulative pseudo-regret $R_T$ as a function of the horizon $T$.

\paragraph{Setup.}
We generate i.i.d. contexts $x_t\in\mathbb R^5$ with an intercept coordinate $x_{t,0}=1$ and independent features $x_{t,j}\sim\mathrm{Unif}(0,1)$ for $j=1,\ldots,4$. The parameter is $\theta^\star=(2,0.125,0.125,0.125,0.125)^\top$, so that $u_t=\langle x_t,\theta^\star\rangle\in(2,2.5)$. We set the noise support to $[-1,1]$ and consider two noise distributions: uniform noise, $\xi_t\sim\mathrm{Unif}[-1,1]$, whose survival function is Lipschitz; and cliff noise, where $\xi_t=0$ with probability $0.3$ and otherwise $\xi_t\sim\mathrm{Unif}[-1,1]$, creating a jump discontinuity of height $0.3$ in the survival function at $w=0$.

For CMRUP, we run horizons $T\in\{500,1000,2000,\ldots,1{,}024{,}000\}$ with $10$ independent seeds per horizon. In the experiments we set $\Delta=\Delta_{\rm mult}\sqrt{d\log(T)/T_1}$ with $\Delta_{\rm mult}=0.35$. This multiplier controls only the residual-grid scale; the conservative markdown coefficient remains fixed at $3$, so Stage-2/3 probes use a $3\Delta$ markdown. For D2-EXP4, we use the same synthetic streams and pseudo-regret metric, running horizons $T\in\{500,1000,2000,\ldots,256{,}000\}$ with $5$ independent seeds per horizon. Since exact enumeration of the D2-EXP4 policy class is computationally prohibitive at these horizons, we use the sampled-policy implementation described in Appendix~\ref{app:exp-details}. All experiments use CPU-only local computations.

\paragraph{Final-regret scaling.}
Figure~\ref{fig:loglog} plots $R_T$ against $T$ on log-log axes. For each algorithm and noise distribution, we fit a power law $R_T\approx C T^\alpha$ by ordinary least squares in log-log space using all available horizons for that algorithm.

On the \textbf{uniform noise} instance, CMRUP has fitted exponent $0.67$, while D2-EXP4 has fitted exponent $0.78$. Although the two algorithms have comparable regret at small horizons, D2-EXP4 exhibits a visibly steeper growth rate, whereas CMRUP tracks the predicted $T^{2/3}$ scaling. On the \textbf{cliff noise} instance, the separation is more pronounced: CMRUP has fitted exponent $0.56$, while D2-EXP4 has fitted exponent $0.77$. This is consistent with the role of the conservative markdown and redirect rule, which prevent the atom-induced jump from causing a worse asymptotic scaling.

\paragraph{Scope.}
These simulations are rate-verification experiments rather than a comprehensive empirical benchmark. The D2-EXP4 baseline is a sampled-policy implementation of the prior agnostic reduction, included to illustrate the empirical scaling difference under the same synthetic environment; it is not an exhaustive implementation of the exponentially large policy class. We do not include ablations of the markdown coefficient or the redirect rule. Such ablations would further isolate the algorithmic mechanism, but the present experiments focus on the main rate comparison under smooth and discontinuous demand.

\section{Discussion and Conclusion}
\label{sec:conclusion}
We close with limitations, ethical considerations, and a summary of the main result. Additional related discussions appear in Appendix~\ref{sec:app_more_discussion}.

\paragraph{Limitations and extensions.}
Our theorem is stated for bounded-support noise, i.i.d. well-conditioned contexts, and an interior price range. These conditions are used for finite-grid learning, Stage-1 parameter estimation, and avoiding boundary clipping in the markdown probes, respectively; they are not continuity assumptions on the demand curve. Extensions to light-tailed noise, horizon-free operation, boundary-clipped probes and richer context processes are discussed in Appendix~\ref{sec:app_more_discussion}.

\paragraph{Computation.}
The direct implementation is polynomial time: Stage~1 solves a $d$-dimensional least-squares problem, and Stage~3 scans $M=\widetilde O(T^{1/3})$ residual-grid points per round. This gives total time $\widetilde O(T^{4/3}+T^{2/3}d^2+d^3)$ and memory $\widetilde O(T^{1/3}+d^2)$. Faster score maximization may improve constants in large-scale deployments, but is orthogonal to the regret analysis.

\paragraph{Ethical considerations.}
This paper studies a theoretical online pricing model. The algorithm prices one item or service instance at a time and does not simultaneously offer the same identical good to different buyers at different prices. Nevertheless, contextual pricing can create disparate-impact concerns if contexts contain protected attributes or proxies. Responsible deployment should audit features, monitor realized prices and acceptance rates across groups, and enforce policy or fairness constraints when appropriate. The proposed method is relatively transparent---a linear predictor, a residual grid, and explicit confidence bounds---which can facilitate such audits.

\paragraph{Conclusion.}
We presented a polynomial-time algorithm for contextual dynamic pricing with linear valuations and fully agnostic non-Lipschitz noise. The algorithm combines randomized estimation of the linear component, conservative markdown probes on a residual grid, and an optimistic redirect rule that controls discontinuity-induced bias without assuming smoothness of the demand curve. It achieves $\widetilde O(T^{2/3})$ expected regret, matching known lower bounds up to logarithmic factors. The result shows that atoms and arbitrary jumps in the noise distribution do not preclude optimal regret in contextual pricing, provided the algorithm uses a pricing rule that respects the one-sided structure of censored feedback.

\newpage
\bibliographystyle{plainnat}
\bibliography{ref}

\newpage
\appendix
\section{More Related Works}
\label{sec:app_more_related_works}

This appendix complements Section~\ref{sec:related} with additional background on dynamic pricing, bandit methods, and fairness-aware pricing.

\paragraph{Semiparametric and distribution-free contextual pricing.}
Several papers study contextual pricing with unknown residual distributions under different structural assumptions. \citet{shah2019semi} consider semiparametric dynamic contextual pricing and jointly learn regression parameters and residual distributions. \citet{luo2021distribution} study contextual dynamic pricing with known $\thst$ and unknown market noise, while \citet{liu2021optimal} and \citet{mao2018contextual} analyze contextual pricing/search models under different assumptions on buyers and feedback. These works are conceptually related because they reduce aspects of pricing to learning residual or threshold structure from censored feedback. Our setting focuses specifically on the linear-valuation model with bounded but otherwise arbitrary noise, and the main challenge is to obtain the minimax $T^{2/3}$ rate without continuity of the residual demand curve.

\paragraph{Classical non-contextual dynamic pricing.}
The non-contextual version of the problem has a single unknown demand curve and no covariates. \citet{kleinberg2003value} established the classical $\Omega(T^{2/3})$ lower bound for online posted-price auctions. \citet{besbes2009dynamic} study dynamic pricing without knowing the demand function and develop near-optimal policies in nonparametric settings. This literature motivates the $T^{2/3}$ benchmark that also appears in contextual pricing. Contextual pricing is harder because the oracle price shifts with $x_t$ and the learner observes only binary feedback, so the linear component and the residual demand curve must be learned simultaneously.

\paragraph{Parametric and high-dimensional pricing.}
A separate line of work studies dynamic pricing under parametric demand models or high-dimensional contextual structure. Examples include settings with noise-free \citep{leme2018contextual}, small noise \citep{cohen2020feature_journal}, known noise distributions \citep{javanmard2019dynamic,xu2021logarithmic, javanmard2020multi}, parametric variance \citep{ban2021personalized}, heteroscedasticity \citep{xu2024pricing} and non-stationarity \citep{baby2023non}. These results obtain strong guarantees by imposing parametric structure, sparsity, or known regularity conditions. Our model is nonparametric in the residual noise distribution and allows discontinuities in demand.


\paragraph{Online pricing with demand learning: A broader view.}
A substantial body of work studies dynamic pricing when the demand function must be learned online under various structural assumptions. \citet{besbes2015surprising} show that even misspecified linear demand models can yield near-optimal pricing policies, highlighting the surprising sufficiency of simple parametric models. \citet{keskin2014dynamic} design semi-myopic policies for single-product pricing under unknown parametric demand and prove asymptotic optimality up to logarithmic factors, while \citet{keskin2017chasing} extend the analysis to non-stationary environments and characterize the additional regret cost of tracking a drifting demand curve. On the algorithmic side, \citet{cheung2017dynamic} study pricing under a strict cap on the number of price changes and obtain matching iterated-logarithm regret bounds. \citet{nambiar2019dynamic} address model misspecification directly and propose a random price shock policy with provable robustness guarantees. \citet{bu2020online} characterize the value of offline data for online pricing through a phase-transition phenomenon and an inverse-square law governing how sample size, location, and dispersion affect the optimal regret. \citet{bastani2022meta} take a multi-task perspective and show that transfer learning across related pricing experiments via a shared Thompson-sampling prior can substantially accelerate learning.  \citet{xu2025dynamic} study a pricing problem with censored demand feedback and compressed demand curve, proposing a UCB-based method to achieve optimal regret. \citet{cohen2018dynamic} develop data-driven pricing methods using sample-average and max-min ratio formulations under limited demand information. The joint pricing-and-inventory variant is studied by \citet{javanmard2017perishability}, who introduces a projected stochastic gradient pricing policy for high-dimensional contextual demand with time-varying parameters. Pricing under inventory constraints is also studied by \citet{chen2019coordinating, chen2020data, chenbx2021nonparametric, chen2023optimal, xu2025joint} under a variety of fully- or partially-observed settings. These works complement the present paper by studying the learning-and-earning trade-off from different angles---misspecification robustness, limited experimentation, offline warm-starting, and meta-learning---whereas our focus is on removing regularity assumptions on the noise distribution altogether.

\paragraph{Posted-price mechanisms and bandit-based pricing.}
Contextual search is closely related to dynamic pricing: both involve learning an unknown threshold from binary feedback.  \citet{lobel2018multidimensional} introduce the Projected Volume algorithm for $d$-dimensional contextual binary search and obtain $O(d\log(d/\varepsilon))$ regret, with direct applications to feature-based pricing. \citet{leme2018contextual} develop intrinsic-volume potential arguments that yield $O_d(1)$ total loss for symmetric loss and doubly logarithmic regret for the pricing variant of contextual search. Robustness to model misspecification in contextual search is studied by \citet{krishnamurthy2020contextual}, who design corruption-robust algorithms whose regret degrades gracefully with the number of irrational agents, and by \citet{leme2022corruption}, who prove tight regret bounds via density-update methods. On the privacy front, \citet{chen2022privacy} propose differentially private contextual pricing policies with near-optimal regret. \citet{wang2021uncertainty} develop confidence-interval-based uncertainty quantification for demand parameters in contextual pricing. In incentive-aware settings, \citet{golrezaei2019dynamic} design contextual robust pricing policies that remain effective when buyers behave strategically. For non-contextual posted-price auctions, \citet{cesa2019dynamic} study revenue maximization when valuations are supported on finitely many unknown points and establish tight $\Theta(\sqrt{KT})$ regret, while \citet{cesa2014regret} give efficient $\widetilde{O}(\sqrt{T})$-regret algorithms for reserve-price optimization in second-price auctions. The closely related bilateral-trade setting is studied by \citet{cesa2024bilateral}, who provide a complete characterization of regret rates across different feedback and valuation models. These works collectively show that the interplay between censored feedback, context dimensionality, strategic behavior, and distributional assumptions governs the statistical complexity of online pricing; our contribution focuses specifically on removing distributional regularity in the residual demand curve while retaining the optimal $\widetilde{O}(T^{2/3})$ rate.

\paragraph{Contextual search and threshold feedback.}
Contextual search and pricing with threshold feedback are closely related to dynamic pricing because the learner observes binary information about an unknown threshold. \citet{mao2018contextual} study contextual pricing for Lipschitz buyers, \citet{liu2021optimal} analyze optimal contextual pricing and extensions, and \citet{leme2022corruption} study model misspecification in contextual search and prove a tight bound. These works differ in model details and regularity assumptions as their non-continuity assumptions are mainly made on the observations (the ``demand'') instead of distributions (the ``demand curve''). That said, they share the challenge of learning from one-bit threshold observations.

\paragraph{Feature-based and policy-class approaches.}
Feature-based pricing algorithms often reduce the problem to learning over a class of policies. This approach is flexible because it can handle weak assumptions on the noise distribution, but it may require a large policy class. \citet{xu2022towards} use this perspective for agnostic linear-valuation pricing and obtain $\widetilde O(T^{3/4})$ regret via a discretized policy class and an adversarial bandit algorithm. Our algorithm instead exploits the additive linear-valuation structure more directly: once the linear component is estimated, the remaining uncertainty is one-dimensional in the residual coordinate.

\paragraph{Semiparametric residual learning.}
\citet{shah2019semi} study semiparametric dynamic contextual pricing, where both regression parameters and the residual distribution must be learned from censored feedback. \citet{fan2024policy} develop a uniformly randomized pricing estimator that yields a semiparametric route to learning the linear component; we use the same identity in Stage~1. The novelty of our analysis lies in how the residual demand curve is handled after parameter learning: conservative markdown produces one-sided grid observations, and redirect-UCB controls discontinuities without imposing smoothness.

\paragraph{Shape constraints and discontinuous demand.} The survival function $S(w)=\mathbb P(\xi\ge w)$ is monotone even when it is discontinuous. Shape-constrained ideas are therefore natural for agnostic pricing. \citet{bracale2025dynamic} use isotonic regression in the linear-valuation model under H\"older-type regularity assumptions, while \citet{gong2026minimax} study contextual dynamic pricing with general valuation models under differentiability conditions on the demand curve. Our analysis uses only monotonicity and bounded support. The markdown sandwich ensures that binary observations can be interpreted as one-sided probes of adjacent survival levels, and the confidence-bound analysis does not require continuity or differentiability of $S$.

\paragraph{Confidence-bound methods.}
Upper-confidence-bound algorithms are a standard tool for balancing exploration and exploitation in multi-armed bandits \citep{auer2002finite,auer2002using}. Our Stage~3 analysis uses the same counting principle: if an arm with confidence radius $b_j$ is selected repeatedly, its count increases and the radius shrinks, and the cumulative confidence cost is bounded by $\widetilde O(\sqrt{MT})$ over $M$ grid indices. The pricing-specific difficulty is that probing a residual index near a discontinuity may introduce a one-sided jump bias. The redirect rule separates this jump cost from the confidence cost and bounds it using the monotonicity identity $\sum_j\alpha_j\le1$. A mix use of upper-confidence-bound and lower-confidence-bound methods is seen in \citet{xu2025online} where they study an online resource allocation problem under uncertainties.

\paragraph{Fairness-aware pricing.}
Personalized and contextual pricing can raise fairness concerns, especially when contexts contain sensitive information or proxies. A growing literature studies dynamic pricing under fairness constraints or fairness regulation \citep{cohen2021dynamic,cohen2022price,chen2021fairness,yang2022fairness, xu2023doubly}. Our regret analysis does not impose fairness constraints, but the algorithmic structure is compatible with them: one can restrict the admissible price set, constrain the context features, or post-process the selected price to satisfy externally specified fairness rules. Appendix~\ref{sec:app_more_discussion} discusses ethical considerations in more detail.

\section{More Discussions}
\label{sec:app_more_discussion}

This appendix expands on the limitations and implementation considerations discussed in Section~\ref{sec:conclusion}.

\paragraph{Overall: Limitations and extensions.}
The analysis assumes i.i.d. contexts with a well-conditioned covariance matrix. This condition is used to learn the linear parameter in Stage~1. Extending the result to richer context processes is an important direction. For example, with obliviously chosen contexts satisfying an empirical full-rank condition, one may still expect the randomized-price estimator to recover $\theta^\star$; handling fully adaptive or adversarial contexts would require a different parameter-learning component.

We also assume bounded noise and a price buffer ensuring that the markdown probes remain in the admissible price interval. Bounded support is standard in worst-case regret analyses for posted-price learning and can often be approximated by truncation under sub-Gaussian or sub-exponential tails. The buffer condition is a boundary condition for the price interval rather than a smoothness condition on demand. Removing it would require explicitly treating clipped markdown probes; this appears technically feasible but would add boundary cases to the sandwich argument.

\paragraph{Computational complexity.}
The algorithm is polynomial time. Stage~1 requires forming and inverting a $d\times d$ covariance matrix, with cost $O(T_1d^2+d^3)$ using a direct implementation. Stage~2 performs $T_w=\widetilde O(T^{2/3})$ constant-time grid updates. Stage~3 scans $M=\widetilde O(T^{1/3})$ grid indices per round to compute the optimistic score, giving total time $\widetilde O(TM)=\widetilde O(T^{4/3})$, plus $O(1)$ empirical-mean updates per round. The overall direct implementation cost is therefore
\[
    \widetilde O(T^{4/3}+T^{2/3}d^2+d^3),
\]
with memory $\widetilde O(T^{1/3}+d^2)$. This is polynomial in both $T$ and $d$ and avoids enumerating a high-dimensional policy class. For very large-scale deployments, faster score maximization would be useful; this computational optimization is orthogonal to the regret analysis.

\paragraph{Bounded support and truncation.}
The bounded-support assumption on $\xi_t$ allows the algorithm to work on a finite residual grid. If the noise has light tails, a truncated version of the algorithm can be applied on an interval whose length grows logarithmically with $T$. The regret contribution from tail events can then be controlled by the tail probability, while the grid size and markdown scale acquire only logarithmic factors. A complete treatment would require replacing the bounded valuation assumption in Stage~1 by a corresponding concentration condition for the randomized-price regression signal.

\paragraph{Price-buffer and clipping.}
The price-buffer assumption ensures that the conservative markdown probes are not clipped in Stages~2--3 on the Stage~1 good event. This keeps the sandwich proof simple: the effective residual of probe $j$ lies exactly in $[w_{j-1},w_j]$. The condition is not a regularity assumption on $F$ or $S$; it only keeps the probed prices away from the boundary of the admissible price interval. For the asymptotic theorem, it is enough that the buffer dominates the grid scale, i.e., $\kappa\gtrsim\Delta$. A fixed positive buffer is a clean sufficient condition.

Boundary clipping can be handled by explicit conventions. If a lower-clipped price is zero and valuations are nonnegative, purchase occurs with probability one, which is consistent with the lower boundary survival value. Upper clipping is ruled out by the bounded valuation model. One can alternatively discard lower-clipped observations from the residual-grid dataset and charge the clipped rounds directly as boundary rounds. We do not include this extension in the main theorem in order to keep the presentation focused on the core discontinuity issue rather than boundary bookkeeping.

\paragraph{Horizon dependence.}
The grid scale $\Delta$ and the lengths of Stages~1--2 depend on the horizon $T$. This is common in regret-optimal pricing algorithms. A horizon-free version can be obtained by a doubling trick: run the algorithm in episodes of lengths $1,2,4,\ldots$ and restart the grid at the beginning of each episode. If the regret in an episode of length $2^r$ is at most $C(2^r)^{2/3}\operatorname{polylog}(2^r)$, then over all completed episodes up to horizon $T$,
\[
    \sum_{r\le \log_2 T} C(2^r)^{2/3}\operatorname{polylog}(2^r)
    \le
    C' T^{2/3}\operatorname{polylog}(T).
\]
Thus the doubling trick preserves the $\widetilde O(T^{2/3})$ rate.

\paragraph{Context assumptions.}
The i.i.d. context assumption is used to obtain a high-probability least-squares estimate of $\theta^\star$ from uniformly randomized prices. If the learner is given an exogenous exploration design with a well-conditioned empirical covariance matrix, the same proof can be adapted. Removing both stochasticity and empirical full-rank conditions is more difficult, because the binary feedback may not contain enough information to identify the linear component in directions that are rarely explored by the contexts.

\paragraph{Choice of the price range in Stage~1.}
The randomized-price regression identity in Stage~1 uses prices uniformly distributed on an interval containing the valuation support. The variance of the regression signal $z_t=B o_t$ scales with $B^2$, so a tighter valid price range can improve finite-sample performance. The theory uses a known upper bound $B$ for simplicity. In practice, one may use the smallest reliable admissible range, or combine the method with a preliminary range-estimation step when conservative bounds are available.

\paragraph{Grid size and dependence on the support radius.}
The confidence cost in Stage~3 scales as $\widetilde O(\sqrt{MT})$, where $M=O(c/\Delta)$. Thus the support radius $c$ enters the problem-dependent constants. This dependence is natural for a uniform residual grid over the entire bounded support. Adaptive grid refinement or focusing on promising residual regions may improve constants in benign instances, but such refinements would need additional bookkeeping to preserve optimism under discontinuities.

\paragraph{Ethical considerations.}
The algorithm prices one item or service instance at a time. It does not simultaneously offer the same identical good to different buyers at different prices, and therefore does not instantiate simultaneous same-good price discrimination. The purpose of the context is to model heterogeneity across sequential pricing instances, such as differentiated products, changing service conditions, or buyer-product pairs.

At the same time, contextual pricing systems should be deployed with explicit governance. If the context vector includes protected attributes or strong proxies, any revenue-maximizing algorithm may produce undesirable disparities over time. Practical deployments should therefore audit the feature set, monitor realized prices and acceptance rates across relevant groups, and enforce policy constraints when required. The transparency of the proposed algorithm---a linear predictor, an explicit residual grid, and confidence-bound decisions---can support such audits more readily than opaque black-box pricing models.

\section{Proofs for the Redirect-UCB Analysis}
\label{sec:app-proofs}

This appendix gives the full proof of Theorem~\ref{thm:main}. We use the notation of Sections~\ref{sec:setup}--\ref{sec:alg}. Throughout, $C,C',C_B$ denote positive problem-dependent numerical constants whose values may change from line to line but do not depend on $T$.

\subsection{Good events}

Recall the Stage~1 prediction event
\[
    \mathcal E_\theta
    :=
    \left\{
        \max_{1\le t\le T}|\hat u_t-u_t|\le \D
    \right\}.
\]
For each queried index $j$, let $\mathcal T_j(t)$ be the set of all rounds before $t$ in which probe action $j$ was actually played. This includes Stage~2 warmup probes, Stage~3 direct probes, and Stage~3 redirected probes from $j+1$ to $j$. Let
\[
    n_j(t):=|\mathcal T_j(t)|,
    \qquad
    \hat m_j(t):=
    \begin{cases}
    n_j(t)^{-1}\sum_{s\in\mathcal T_j(t)}o_s,& n_j(t)>0,\\
    0,& n_j(t)=0.
    \end{cases}
\]
For $n_j(t)>0$, define the predictable average
\[
    \bar m_j(t):=
    \frac{1}{n_j(t)}
    \sum_{s\in\mathcal T_j(t)}
    \E[o_s\mid \Hist_{s-1},x_s,\text{probe }j].
\]
If $n_j(t)=0$, set $\bar m_j(t)=0$. Define
\[
    b_j(t):=C_{\rm ucb}\sqrt{\frac{\log T}{\max\{1,n_j(t)\}}}.
\]
Let $\mathcal E_{\rm conf}$ denote the event that, simultaneously for every Stage~3 round $t$ and every $j\in\mathcal I$,
\[
    |\hat m_j(t)-\bar m_j(t)|\le \frac12 b_j(t).
\]

\subsection{Proof of Lemma~\ref{lem:stage1}}
\label{app:stage1-proof}

\begin{proof}
In Stage~1, prices are sampled independently as $p_t\sim\mathrm{Unif}[0,B]$. Since $y_t\in[0,B]$ almost surely under Assumption~\ref{ass:bdd},
\[
    \E[o_t\mid x_t,\xi_t]
    =
    \Pp(p_t\le y_t\mid x_t,\xi_t)
    =
    \frac{y_t}{B}.
\]
Thus, with $z_t:=B o_t$,
\[
    \E[z_t\mid x_t]
    =
    \E[y_t\mid x_t]
    =
    \langle x_t,\thst\rangle+
    \E[\xi_t\mid x_t]
    =
    \langle x_t,\thst\rangle,
\]
where the last equality uses Assumption~\ref{ass:bdd}. Therefore the Stage~1 regression model is
\[
    z_t=x_t^\top\thst+\eta_t,
    \qquad
    \E[\eta_t\mid x_t]=0,
\]
with bounded noise $|\eta_t|\le 2B$.

Let
\[
    \widehat\Sigma_1:=\frac1{T_1}\sum_{t=1}^{T_1}x_tx_t^\top,
    \qquad
    \Sigma:=\E[x_tx_t^\top].
\]
By matrix Bernstein for bounded random-design covariates, for $T_1\gtrsim (B_x^4/\lambda_0^2)d\log T$,
\[
    \lambda_{\min}(\widehat\Sigma_1)\ge \lambda_0/2
\]
with probability at least $1-T^{-3}$. On this event, the Moore--Penrose estimator in Algorithm~\ref{alg:main} equals the usual inverse estimator, and
\[
    \hth-\thst
    =
    \widehat\Sigma_1^{-1}
    \left(\frac1{T_1}\sum_{t=1}^{T_1}x_t\eta_t\right).
\]
A vector Bernstein or coordinate-wise Hoeffding argument gives, with probability at least $1-T^{-3}$,
\[
    \left\|\frac1{T_1}\sum_{t=1}^{T_1}x_t\eta_t\right\|_2
    \le
    C B B_x\sqrt{\frac{d\log T}{T_1}}.
\]
Combining the two displays,
\[
    \|\hth-\thst\|_2
    \le
    \frac{C B B_x}{\lambda_0}
    \sqrt{\frac{d\log T}{T_1}}.
\]
Thus
\[
    \max_{1\le t\le T}|\hat u_t-u_t|
    \le
    B_x\|\hth-\thst\|_2
    \le
    C\frac{B B_x^2}{\lambda_0}
    \sqrt{\frac{d\log T}{T_1}}
    \le \D
\]
by the choice of $C_\theta$. A union bound over the covariance and noise events proves the lemma.
\end{proof}

\subsection{Proof of Lemma~\ref{lem:markdown-sandwich}}
\label{app:sandwich-proof}

\begin{proof}
On $\mathcal E_\theta$, $|\hat u_t-u_t|\le\D$. When probe index $j$ is played in Stages~2--3, Assumption~\ref{ass:buffer} ensures that clipping does not occur, so
\[
    p_t=\hat u_t+w_{j+1}-3\D.
\]
Since $w_{j+1}=w_j+2\D$ and $w_{j-1}=w_j-2\D$,
\[
    p_t-u_t
    =
    w_{j+1}-3\D+(\hat u_t-u_t)
    \in
    [w_j-2\D,w_j]
    =
    [w_{j-1},w_j].
\]
Because $S$ is non-increasing,
\[
    S(w_j)
    \le
    S(p_t-u_t)
    \le
    S(w_{j-1}).
\]
The middle term equals $\E[o_t\mid\Hist_{t-1},x_t,\text{probe }j]$ by the definition of $S$, proving the claim.
\end{proof}

\subsection{Proof of Lemma~\ref{lem:raw-confidence}}
\label{app:confidence-proof}

\begin{proof}
Fix a queried index $j$. Let $\tau_{j,m}$ be the stopping time of the $m$-th play of probe action $j$, where a play may arise from Stage~2, a direct Stage~3 probe, or a redirected Stage~3 probe from $j+1$ to $j$. Define
\[
    Y_{j,m}:=o_{\tau_{j,m}},
    \qquad
    \mu_{j,m}:=
    \E[o_{\tau_{j,m}}\mid \Hist_{\tau_{j,m}-1},x_{\tau_{j,m}},\text{probe }j].
\]
Then $Y_{j,m}-\mu_{j,m}$ is a bounded martingale-difference sequence with respect to the stopped filtration generated by the probe times. The adaptivity of the probe times is harmless because the action is chosen before observing $o_{\tau_{j,m}}$.

For any fixed $j$ and $m$, Hoeffding--Azuma gives
\[
    \Pp\left(
    \left|\sum_{r=1}^m(Y_{j,r}-\mu_{j,r})\right|
    \ge
    C\sqrt{m\log T}
    \right)
    \le
    2T^{-5}
\]
for a sufficiently large universal constant $C$. Taking a union bound over at most $M\le T$ queried indices and $m\le T$ sample sizes yields, with probability at least $1-T^{-2}$,
\[
    \left|\frac1m\sum_{r=1}^mY_{j,r}
    -
    \frac1m\sum_{r=1}^m\mu_{j,r}\right|
    \le
    \frac12 C_{\rm ucb}\sqrt{\frac{\log T}{m}}
\]
for all $j$ and $m\ge1$, after increasing $C_{\rm ucb}$. This is exactly the event $\mathcal E_{\rm conf}$.

On $\mathcal E_\theta$, Lemma~\ref{lem:markdown-sandwich} gives $S(w_j)\le\mu_{j,m}\le S(w_{j-1})$ for every probe of index $j$. Hence $S(w_j)\le\bar m_j(t)\le S(w_{j-1})$. Combining this with $|\hat m_j(t)-\bar m_j(t)|\le b_j(t)/2$ gives
\[
    \hat m_j(t)+b_j(t)
    \ge
    \bar m_j(t)+\frac12 b_j(t)
    \ge
    S(w_j),
\]
and
\[
    \hat m_j(t)+b_j(t)
    \le
    \bar m_j(t)+\frac32 b_j(t)
    \le
    S(w_{j-1})+2b_j(t),
\]
where the constant $2$ is used for notational simplicity. This proves the lemma.
\end{proof}

\subsection{Proof of Lemma~\ref{lem:optimism}}
\label{app:optimism-proof}

\begin{proof}
Fix a Stage~3 round $t$, and let $w_t^\star\in[-c,c]$ be an optimal residual. Choose $q\in\mathcal I$ such that
\[
    w_t^\star\in[w_q,w_{q+1}].
\]
Such an index exists because $w_0=-c$ and $w_M\ge c$. On $\mathcal E_\theta$,
\[
    \hat u_t+w_{q+1}+\D
    \ge
    u_t+w_{q+1}
    \ge
    u_t+w_t^\star.
\]
Because $S$ is non-increasing and $w_q\le w_t^\star$,
\[
    S(w_q)\ge S(w_t^\star).
\]
By Lemma~\ref{lem:raw-confidence},
\[
    \hat m_q(t)+b_q(t)
    \ge S(w_q)
    \ge S(w_t^\star).
\]
Since $S(w_t^\star)\le1$, the truncation does not hurt this lower bound:
\[
    \min\{1,\hat m_q(t)+b_q(t)\}
    \ge S(w_t^\star).
\]
Therefore
\[
    U_{t,q}
    =
    (\hat u_t+w_{q+1}+\D)
    \min\{1,\hat m_q(t)+b_q(t)\}
    \ge
    (u_t+w_t^\star)S(w_t^\star)
    =V_t^\star.
\]
Taking the maximum over $j$ proves the claim.
\end{proof}

\subsection{Proof of Lemma~\ref{lem:one-round}}
\label{app:one-round-proof}

\begin{proof}
By Lemma~\ref{lem:optimism} and the choice of $j_t$,
\[
    V_t^\star\le U_{t,j_t}.
\]
We compare this selected score with the revenue of the action actually played.

\paragraph{Direct-probe mode.}
Let $j=j_t$. The algorithm plays probe $j$, so
\[
    p_t=\hat u_t+w_{j+1}-3\D,
    \qquad
    a_{t,j}:=\hat u_t+w_{j+1}+\D=p_t+4\D.
\]
Let $m_{t,j}$ be the conditional purchase probability of this played action. Since $\min\{1,x\}\le x$ for all $x$,
\[
    U_{t,j}
    \le
    (p_t+4\D)(\hat m_j+b_j).
\]
On $\mathcal E_{\rm conf}$ and by Lemma~\ref{lem:markdown-sandwich},
\[
    \hat m_j+b_j
    \le
    S(w_{j-1})+2b_j.
\]
Moreover $m_{t,j}\ge S(w_j)$, so
\[
    S(w_{j-1})=S(w_j)+\alpha_j\le m_{t,j}+\alpha_j.
\]
Thus
\[
    \hat m_j+b_j
    \le
    m_{t,j}+\alpha_j+2b_j.
\]
Therefore
\begin{align*}
    V_t^\star-\pi(x_t,p_t)
    &\le
    U_{t,j}-p_tm_{t,j} \\
    &\le
    (p_t+4\D)(m_{t,j}+\alpha_j+2b_j)-p_tm_{t,j} \\
    &\le
    4\D+(B+4\D)\alpha_j+2(B+4\D)b_j.
\end{align*}
For all sufficiently large horizons $\D\le1$; constants are absorbed into $C$ and $C_B$. This proves
\[
    V_t^\star-\pi(x_t,p_t)
    \le C\D+B\alpha_j+C_B b_j.
\]

\paragraph{Redirect mode, $j_t\ge1$.}
Let $j=j_t\ge1$. The algorithm plays probe $j-1$, so
\[
    p_t=\hat u_t+w_j-3\D,
    \qquad
    a_{t,j}:=\hat u_t+w_{j+1}+\D=p_t+6\D.
\]
Let $m^-_{t,j}$ be the conditional purchase probability of the played probe $j-1$. By Lemma~\ref{lem:markdown-sandwich},
\[
    m^-_{t,j}\ge S(w_{j-1}).
\]
For the selected score at index $j$, Lemma~\ref{lem:raw-confidence} gives
\[
    \hat m_j+b_j\le S(w_{j-1})+2b_j\le m^-_{t,j}+2b_j.
\]
Again using $\min\{1,x\}\le x$,
\begin{align*}
    V_t^\star-\pi(x_t,p_t)
    &\le
    U_{t,j}-p_tm^-_{t,j} \\
    &\le
    (p_t+6\D)(m^-_{t,j}+2b_j)-p_tm^-_{t,j} \\
    &\le
    6\D+2(B+6\D)b_j.
\end{align*}
Redirect mode occurs only when $b_j\le\D$, so the last display is at most $C_B\D$.

\paragraph{Boundary case $j_t=0$.}
When $j_t=0$, the algorithm plays probe $0$. This is the direct-probe argument with
\[
    \alpha_0=S(w_{-1})-S(w_0)=1-1=0,
\]
because $w_0=-c$ and $S(w_0)=1$ under the support convention. In redirect mode, $b_0\le\D$, so the regret is again at most $C_B\D$.
\end{proof}

\subsection{Proof of Lemma~\ref{lem:stage3-good}}
\label{app:cumulative-proof}

\begin{proof}
Let $\mathcal P$ be the set of Stage~3 rounds in direct-probe mode and $\mathcal R$ the set of Stage~3 rounds in redirect mode or the boundary case. By Lemma~\ref{lem:one-round}, the redirect rounds contribute at most
\[
    \sum_{t\in\mathcal R} C_B\D
    \le
    C_BT\D
    =
    \widetilde O(T^{2/3}).
\]

For direct-probe rounds, Lemma~\ref{lem:one-round} gives three terms. The grid-scale term is
\[
    \sum_{t\in\mathcal P} C\D
    \le
    CT\D
    =
    \widetilde O(T^{2/3}).
\]
For the confidence-radius term, each direct probe of index $j$ increments $n_j$. Hence
\begin{align*}
    \sum_{t\in\mathcal P} b_{j_t}(t)
    &\le
    C_{\rm ucb}\sqrt{\log T}
    \sum_{j=0}^{M-1}
    \sum_{m=1}^{n_j(T)}\frac1{\sqrt m} \\
    &\le
    2C_{\rm ucb}\sqrt{\log T}
    \sum_{j=0}^{M-1}\sqrt{n_j(T)} \\
    &\le
    2C_{\rm ucb}\sqrt{\log T}
    \sqrt{M\sum_j n_j(T)} \\
    &\le
    \widetilde O(\sqrt{MT})
    =
    \widetilde O(T^{2/3}),
\end{align*}
where $M=O(\D^{-1})=\widetilde O(T^{1/3})$.

For the adjacent-jump term, index $j$ is directly probed only while
\[
    b_j(t)=C_{\rm ucb}\sqrt{\frac{\log T}{\max\{1,n_j(t)\}}}>\D.
\]
Thus the number of direct probes of any fixed index $j$ is at most $C_{\rm ucb}^2\log T/\D^2+1$. Therefore
\[
    \sum_{t\in\mathcal P}\alpha_{j_t}
    \le
    \left(C\frac{\log T}{\D^2}+1\right)
    \sum_{j=0}^{M-1}\alpha_j
    \le
    C\frac{\log T}{\D^2}+1
    =
    \widetilde O(T^{2/3}),
\]
using $\sum_j\alpha_j\le1$. Combining the three direct-probe terms and the redirect contribution proves the lemma.
\end{proof}

\subsection{Proof of Theorem~\ref{thm:main}}
\label{app:main-proof}

\begin{proof}
Stages~1 and~2 each last $O(T^{2/3})$ rounds and incur per-round regret at most $B$, so their expected contribution is $O(T^{2/3})$.

Let
\[
    \mathcal E:=\mathcal E_\theta\cap\mathcal E_{\rm conf}.
\]
By Lemmas~\ref{lem:stage1} and~\ref{lem:raw-confidence}, $\Pp(\mathcal E^c)\le O(T^{-2})$. On $\mathcal E$, Lemma~\ref{lem:stage3-good} gives Stage~3 regret $\widetilde O(T^{2/3})$. On $\mathcal E^c$, total regret is at most $BT$, so the failure-event contribution is at most $BT\cdot O(T^{-2})=O(1/T)$. Therefore
\[
    R_T=\E[\mathcal R_T]
    \le
    \widetilde O(T^{2/3}).
\]
\end{proof}
\begin{remark}
    The proof yields the following informal decomposition, suppressing logarithmic
factors and problem-dependent constants:
\[
R_T
\;\lesssim\;
B(T_1+T_w)
+
T\Delta
+
\sqrt{MT\log T}
+
\frac{\log T}{\Delta^2},
\]
where \(M=O(c/\Delta)\) is the number of queried residual-grid indices and
\[
\Delta
=
C_\theta
\frac{B B_x^2}{\lambda_0}
\sqrt{\frac{d\log T}{T_1}}.
\]
The four terms correspond respectively to the two initial stages, grid-scale loss,
UCB confidence accumulation, and the cumulative adjacent-jump cost. Taking
\(T_1=T_w=\lceil T^{2/3}\rceil\) and \(M=O(\Delta^{-1})\) gives
\(\widetilde O(T^{2/3})\).
\end{remark}

\subsection{Lower-bound alignment under the assumptions of this paper}
\label{app:lower-bound-alignment}

We briefly justify that the assumptions used in Theorem~\ref{thm:main} do not remove the standard \(T^{2/3}\) lower-bound barrier. The point is that a non-contextual posted-price hard instance can be embedded into our contextual linear-valuation model while satisfying bounded support, zero-mean noise, stochastic well-conditioned contexts, and the price-buffer condition.

Consider any one-dimensional posted-price lower-bound family supported on an interior valuation interval \([\underline v,\bar v]\subset(0,B)\) for which every pricing algorithm incurs \(\Omega(T^{2/3})\) regret; the construction of \citet{kleinberg2003value} can be placed on such an interval by an affine rescaling of prices and valuations. Let \(V\) denote the valuation in such a hard instance and set \(u_0:=\mathbb E[V]\), \(\xi:=V-u_0\). Then \(\mathbb E[\xi]=0\) and \(\xi\in[-c,c]\) for \(c:=\max\{u_0-\underline v,\bar v-u_0\}\).

Now construct a contextual instance in dimension \(d\ge 2\) as follows. Let
\[
    x_t=(1,Z_{t,2},\ldots,Z_{t,d}),
\]
where \(Z_{t,2},\ldots,Z_{t,d}\) are independent Rademacher random variables, independent across \(t\) and independent of \(\xi_t\). Let
\[
    \theta^\star=(u_0,0,\ldots,0).
\]
Then \(x_t^\top\theta^\star=u_0\), so the buyer valuation is \(y_t=u_0+\xi_t=V_t\), exactly the original non-contextual valuation. The context distribution is i.i.d. and well-conditioned, since \(\mathbb E[x_tx_t^\top]=I_d\). The dummy covariates carry no valuation information, but they ensure that the full-rank context assumption is satisfied.

The bounded-support and normalization assumptions hold by construction. The price-buffer condition also holds whenever the hard instance is placed in an interior interval: \(x_t^\top\theta^\star-c\ge \underline v>0\) and \(x_t^\top\theta^\star+c\le \bar v\le B\). Therefore this contextual instance satisfies Assumptions~\ref{ass:bdd}--\ref{ass:buffer}. Any contextual pricing algorithm for our model, when run on this instance, induces a non-contextual posted-price algorithm with the same revenue process and the same regret. Hence an \(o(T^{2/3})\) regret bound for our model would contradict the non-contextual lower bound.

The contextual lower bound of \citet{xu2022towards} provides a complementary comparison: it shows that the same \(T^{2/3}\) barrier already arises under Lipschitz demand in a contextual linear-valuation setting. The embedding above is sufficient for the formal minimax lower bound under the assumptions of this paper, while the Xu--Wang lower bound places our rate in the broader contextual-pricing landscape.

\subsection{Experimental implementation details}
\label{app:exp-details}
All reported experiments use pseudo-regret computed from the known synthetic survival function. For a posted price $p_t$ and linear value $u_t$, the algorithm's expected revenue is $p_tS(p_t-u_t)$, where throughout the paper $S(w)=\Pr(\xi\ge w)$ uses the inclusive convention matching the purchase event $p_t\le y_t$. The oracle benchmark is $\max_{p\in[0,B]}pS(p-u_t)$. For the uniform-noise instance this maximum is computed from the piecewise-quadratic form of the revenue curve. For the cliff-noise instance, where $\Pr(\xi=0)=0.3$, the candidate set additionally includes the atom price $p=u_t$.

Exact D2-EXP4 requires enumerating a discretized family of valuation parameters and noise distributions, which is computationally prohibitive at the horizons in Figure~\ref{fig:loglog}. We therefore use a sampled-policy implementation. Each sampled policy is indexed by a candidate linear parameter and a candidate monotone survival curve on a residual grid. In our runs, the default ensemble contains $K_{\rm policy}=2048$ policies, sampled from $K_\theta=256$ parameter candidates and $K_F=64$ survival-function candidates. The parameter pool samples the intercept from $[1,3]$ and the remaining coefficients from $[0,0.3]$, together with a few structured anchor vectors. The survival-function pool contains uniform-like, step-like, and cliff-like structured candidates, supplemented by randomly generated monotone survival curves. EXP4 is then run over this sampled ensemble using an action grid with spacing $\gamma=\mathrm{clip}(T^{-1/4},0.02,0.10)$, learning rate $\eta=\min\{0.5,\sqrt{\log(K_{\rm policy})/(T K_{\rm act})}\}$, and explicit exploration probability $\min\{0.2,\sqrt{K_{\rm act}\log(K_{\rm policy})/T}\}$, where $K_{\rm act}$ is the number of price-grid actions. The displayed D2-EXP4 curves should therefore be read as a computationally tractable sampled-policy implementation of the D2-EXP4 reduction, not as an exhaustive implementation of the full exponentially large policy class.

\end{document}